\documentclass[letterpaper]{article}
\usepackage{aaai}
\usepackage{times}
\usepackage{helvet}
\usepackage{courier}
\usepackage{graphicx}
\usepackage{subfigure}
\usepackage{algorithm}
\usepackage{algpseudocode}
\usepackage{amssymb}
\usepackage{epsfig}
\usepackage{amsmath}
\DeclareMathOperator*{\argmin}{argmin}
\usepackage{latexsym}

\usepackage{wrapfig}

\usepackage{graphicx}
\usepackage{bm}
\usepackage{multirow}
\usepackage{subfigure}
\usepackage{epstopdf}
\usepackage{setspace}
\usepackage{epsfig}
\usepackage{epstopdf}
\usepackage{float}
\usepackage{color}
\usepackage{booktabs}
\usepackage{caption}
\usepackage{graphicx}

\usepackage{amsfonts}       
\usepackage{nicefrac}       
\usepackage{microtype}      
\usepackage{amsmath}
\usepackage{amsthm}
\usepackage{graphicx}
\usepackage{bm}
\usepackage{multirow}
\usepackage{subfigure}
\usepackage{epstopdf}
\usepackage{setspace}
\usepackage{epsfig}
\usepackage{epstopdf}
\usepackage{float}
\usepackage{color}
\usepackage{booktabs}
\usepackage{caption}

\usepackage{amssymb}

\newtheorem{myDef}{Definition} 

\newtheorem{myThe}{Theorem}
\newtheorem{myCor}{Corollary}

\title{Sentence-wise Smooth Regularization for Sequence to Sequence Learning}

\date{}

\begin{document}
    
\author{Chengyue Gong$^\dag$\thanks{The work was conducted when the first author was an intern at Microsoft.}, Xu Tan$^\S$, Di He$^\dag$, Tao Qin$^\S$ \\
$^\dag$Peking University \\
$^\S$Microsoft Research \\
$^\S$\{xuta,taoqin\}@microsoft.com, $^\dag$\{cygong,di\_he\}@pku.edu.cn
}

	\maketitle
	\begin{abstract}
    Maximum-likelihood estimation (MLE) is widely used in sequence to sequence tasks for model training. It uniformly treats the generation/prediction of each target token as multi-class classification, and yields non-smooth prediction probabilities: in a target sequence, some tokens are predicted with small probabilities while other tokens are with large probabilities. According to our empirical study, we find that the non-smoothness of the probabilities results in low quality of generated sequences. In this paper, we propose a sentence-wise regularization method which aims to output smooth prediction probabilities for all the tokens in the target sequence. Our proposed method can automatically adjust the weights and gradients of each token in one sentence to ensure the predictions in a sequence uniformly well. Experiments on three neural machine translation tasks and one text summarization task show that our method outperforms conventional MLE loss on all these tasks and achieves promising BLEU scores on WMT14 English-German and WMT17 Chinese-English translation task.

	\end{abstract}
	
	\section{Introduction}
	\label{Introduction}
	Sequence to sequence learning has achieved great success in many natural language processing (NLP) tasks such as machine translation \cite{bahdanau2014neural,cho2014learning,wu2016google,vaswani2017attention,gehring2017convolutional,layerwise,DBLP:conf/coling/SongTHLQL18,DBLP:conf/naacl/ShenTHQL18}, text summarization \cite{rush2015neural,shen2016neural} and dialog systems \cite{serban2016building}. The most popular model architecture for sequence to sequence learning is the encoder-decoder framework, which learns a probabilistic mapping $P(\bm{y}|\bm{x})$ from a source sequence $\bm{x}$ to a target sequence $\bm{y}$.
	
During training, maximum-likelihood estimation (MLE) is widely used to learn the model parameters. Since $\bm{y}=\{y_1, y_2,\cdots,y_n\}$ is a sequence of tokens, by decomposing $P(\bm{y}|\bm{x})$ using the chain rule, MLE becomes to minimize the average of the negative log probabilities of individual target tokens $-\frac{1}{n}\sum^{n}_{i=1} \log P(y_i|\bm{y}_{<i},\bm{x};\theta)$ over all $(\bm{x},\bm{y})$ pairs in training data, where  $\theta$ is the model parameter. While MLE training objective sounds reasonable in general, it has several potential issues in the context of sequence to sequence tasks, such as exposure bias~\cite{wiseman2016sequence,beyond_error,guo2019aaai}, loss-evaluation mismatch ~\cite{wiseman2016sequence,vijayakumar2016diverse}, and myopic bias~\cite{He2017Decoding}. 
	
In this paper, we investigate the MLE loss for sequence to sequence tasks from an optimization and algorithmic perspective. According to our empirical study, we find that for a well-trained model, the probabilities of individual tokens (i.e., per-token probabilities) in the target sequence for a given source sequence sometimes vary dramatically across time steps. We conduct an analysis on WMT14 English-German validation set to explore the relationship between the smoothness\footnote{Here we simply use variance to represent the smoothness for clearness and simplicity at the beginning of the paper. The formal definition of the smoothness obeys the definition in Section~\ref{definition}.} of probabilities and BLEU scores, as shown in Figure~\ref{score-table}. As can be seen, for source-target pairs with same/similar MLE loss, the smoother the per-token probabilities are, the higher BLEU score is. Intuitively speaking, for source-target sequence pairs with the same/similar MLE loss, the sequence pair with non-smooth probabilities indicates that it has more small and large per-token probabilities. First, if the probability of certain target token is relatively low, the model is likely to make a wrong prediction for this token and thus generates a bad sequence during inference. Second, if the probability of certain target token is extremely large, i.e., larger than 0.9, which is enough to ensure correct prediction, there is no need to continue optimizing this token and it's better to leave the efforts of the model to other low-probability tokens. 

\begin{figure*}[t]
\centering 
\subfigure[]
{ 
\label{score-table}
\includegraphics[width=3.2in]{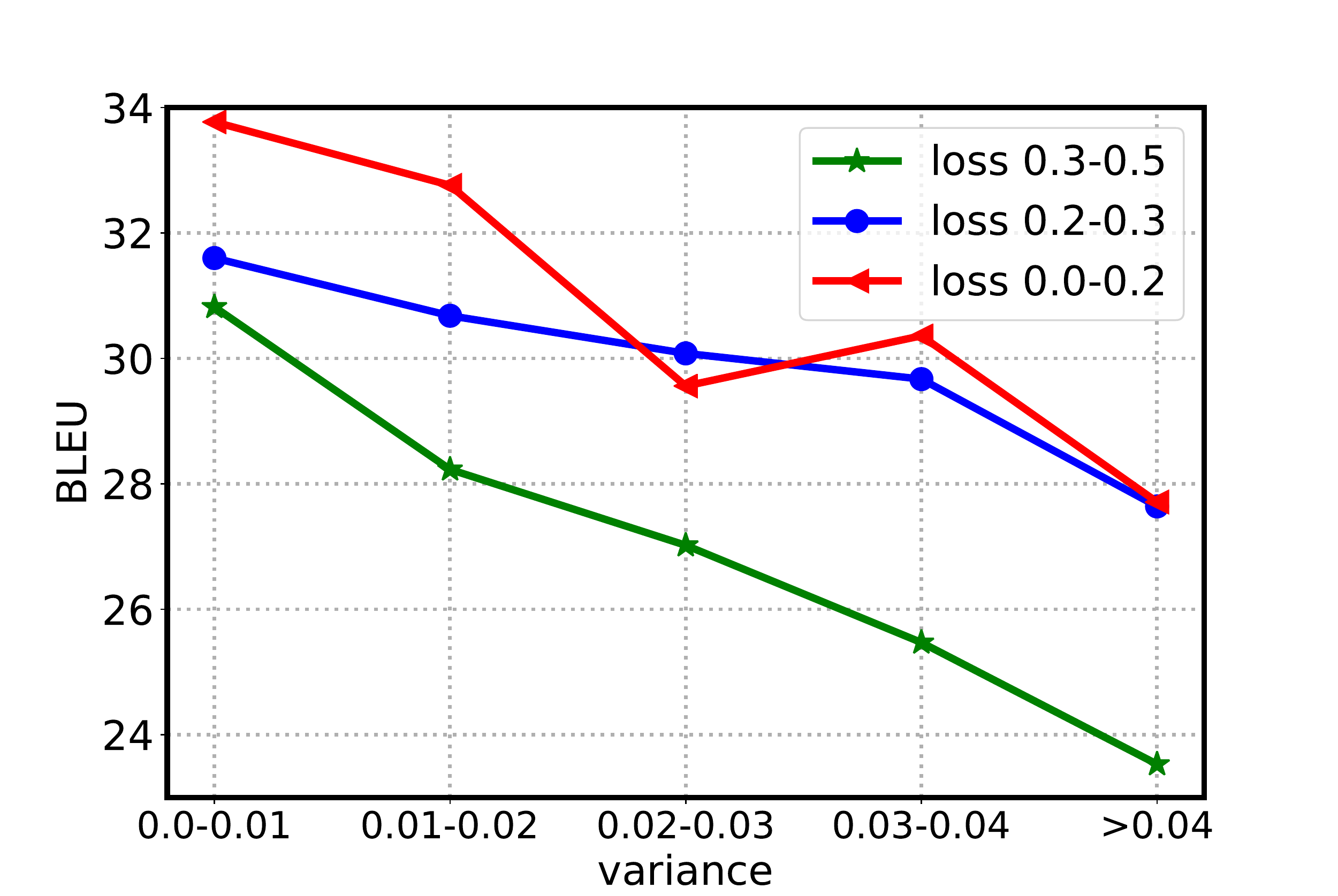}
} 
\subfigure[]
{ 
\label{component}
\includegraphics[width=3.2in]{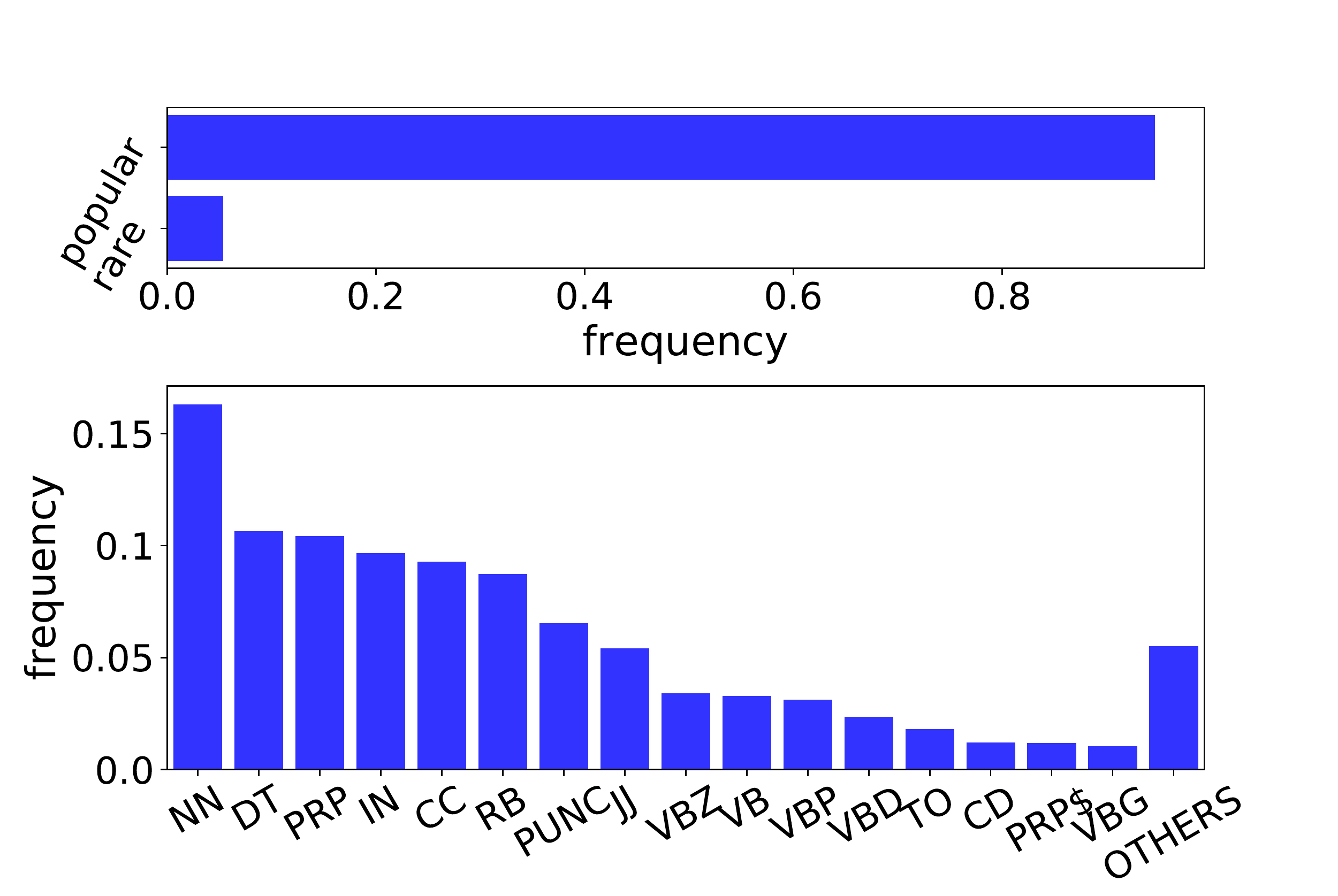}
} 
\caption{(a): The variance of per-token probabilities negatively correlates with BLEU scores. We choose the translation sentences from WMT14 English$\to$German validation set based on a well-trained Transformer model (experimental details are in Section~\ref{experiments}). We partition the sentence pairs into three groups according to the MLE loss: 0-0.2, 0.2-0.3, 0.3-0.5, and divide the sentence pairs in each group into 5 buckets according to the variance of token probabilities: 0-0.01, 0.01-0.02, 0.02-0.03, 0.03-0.04, >0.04. Different colors represent different MLE loss groups.  (b): The components of the tokens with low probability in IWSLT14 German$\to$English training set based on a well-trained Transformer model. We analyze the components from the perspective of popular/rare words as well as different POS tags. More details about these analyses in (a) and (b) can be found in supplementary materials (part A).} 
\label{fig:1} 
\end{figure*}

Besides, we wonder where the non-smooth problem comes from and then study the components of tokens with low probability in the train set. As is well-known, rare words are hard to learn~\cite{gong2018nips} and low probabilities may just come from rare words. However, as shown in Figure~\ref{component}, we find that nearly 6\% of low-probability tokens are rare words, and this ratio almost equals to the ratio of rare words in the whole dataset, which demonstrates low-probability tokens do not prefer  rare words more than popular words.\footnote{We simply set the top $10\%$ frequent words as popular words and denote the rest as rare words. Other reasonable thresholds still come to similar conclusion.} We further partition the low-probability tokens into different POS tags and find that the major part are nouns, determiners, pronouns, conjunctions and punctuations,\footnote{We use Natural Language Toolkit to do POS tag: https://www.nltk.org/.} and ratio of conjunctions and pronouns in low-probability tokens are higher than the ratio in the whole dataset. More detailed discussions are listed in supplementary materials (part A). From these analyses, we find the low-probability tokens are not mainly caused by rare words and thus the non-smooth probabilities cannot be addressed effectively by the conventional methods tackling rare words~\cite{luong2016achieving,DBLP:conf/ijcai/LiZZ16}.  

Inspired by the above observations, in this work, we aim to learn a sequence generation model which can output smooth per-token probabilities across time steps. We first design two principles (Prediction Saturation and Wooden Barrel) for a loss function to ensure the smoothness of per-token probabilities. We then design a sentence-level smooth regularization term adding to the MLE loss to guide the training process. We show that our proposed loss satisfies both principles while the conventional MLE loss does not. The designed regularization term aims at reducing the variance of per-token probabilities across time steps and make those probabilities smoother. When there are tokens with both large and small probabilities in a sequence, the regularizer actually sacrifices the loss of tokens with large target probability only a little (which usually does not hurt the training and inference much), focuses on optimizing losses for the tokens of small probabilities (which may bring advantages), and leads to sequences with smoother per-token probabilities and consequently better performance.
	
To demonstrate the effectiveness of our proposed loss, we conduct experiments on three machine translation tasks and one text summarization task. The results show that our designed loss indeed leads to smooth per-token probabilities and consequently better performance: Our method outperforms baseline methods on all the tasks and sets new records on all the three NMT tasks.
    

\section{Maximum-Likelihood Estimation}	
\label{sec2_mle}
In this section, we briefly describe how MLE works in sequence to sequence learning.
Denote $(\bm{x},\bm{y}) $ as a source-target sequence pair, where $\bm{x}$ is a sequence from the source domain and $\bm{y}$ is a sequence from the target domain. Sequence to sequence learning aims to model the conditional probability $P(\bm{y}|\bm{x}; \theta)$, where $\theta$ is the model parameter. The widely adopted MLE loss $L_{MLE}(\bm{x}, \bm{y}; \theta)$ defined on a sequence pair $(\bm{x}, \bm{y})$ is as follows.
\begin{equation}
\begin{aligned}
L_{MLE}(\bm{x}, \bm{y}; \theta) &= -\log P(\bm{y}|\bm{x}; \theta) 
\\ & =-\sum_{i=1}^{n}\log P(y_{i}|\bm{y}_{<i},\bm{x};\theta),  
\label{seq_loss}
\end{aligned}
\end{equation}	
where $\bm{y}_{<i}$ is the sequence of the preceding tokens before position $i$ and $n$ is the length of the target sequence. For simplicity, we denote $P(y_{i}|\bm{y}_{<i},\bm{x};\theta)$ as $p_i$.
By applying length normalization, we have 
\begin{equation}
\begin{aligned}
\label{eq:mle}
L_{MLE} (\bm{x}, \bm{y}; \theta) = -\frac{1}{n}\sum_{i=1}^n{\log p_{i}},
\end{aligned}	
\end{equation}	
where $p_i$ is a function of $(\bm{x}, \bm{y}, \theta)$. We call the probabilities $(p_1, p_2, \cdots, p_n)$ as the \emph{probability sequence}, which is further denoted as $\bm{p}$. The sub-sequence $(p_i, \cdots, p_j)$ from time step $i$ to $j$ is denoted as $\bm{p}_{i:j}$. In general, the loss is a function of the probability sequence and we denote it as $L(\bm{p})$ or $L(p_1,\cdots,p_n)$. 


\section{Our Method}
\label{method}
In this section, we first introduce two principles that a loss function should satisfy (while MLE can not) to ensure smooth per-token probabilities. We then propose a sentence-wise smooth regularization term, and show that equipped with the regularization term, the MLE loss satisfies these principles.

\subsection{Principles}
\label{2-condition}
Clearly, to generate smooth $\bm{p}$, we need to reduce the variance and difference of per-token probabilities. To achieve this, intuitively, a loss function should guide the optimization procedure to pay less attention to tokens of large probabilities and more attention to tokens of small probabilities. We specify two principles regarding this as follows. 

\noindent{\textbf{Prediction Saturation (PS)}} When the prediction of a token in the target sequence is good enough and much better than other tokens, further improving this token brings no more gain or even negative gain to the overall loss. We call such a principle the \emph{prediction saturation} principle.

\noindent{\textbf{Wooden Barrel (WB)}} As well known, the capacity of a barrel is determined not by the longest wooden bars, but by the shortest. If the prediction probability of a token is small while its adjacent tokens are relatively large, it becomes the barrier during inference and makes the probability sequence non-smooth. We hope the efforts that the optimization algorithm puts to this token positively correlates with the gap between the probability of the token and that of its neighbors. We call such a principle the \emph{wooden barrel} principle.

The above two principles are mathematically characterized as follows.

	\begin{myDef}[PS Principle]
		Suppose $(\beta,\epsilon)$ is a pair of positive constants. We say an objective function $L(p_1, p_2, \cdots, p_n)$ is $(\beta,\epsilon)$-saturated, if  $\forall i, {p'}_i\ge p_i\ge \beta$ and $\max(p_{i-1}, p_{i+1}) \le \beta - \epsilon$, the following inequality holds: $L(\cdots, p_i,\cdots) \le L(\cdots, p'_i,\cdots)$.
	\end{myDef}	    
This principle means that if the prediction of a token is good enough, i.e., its probability is larger than a threshold $\beta$, and its adjacent tokens are not well optimized, i.e., their probabilities are smaller than $\beta-\epsilon$, further improving this good enough token (e.g., pushing $p_i$ to a large value $p'_i$) does not lead to the decrease of the loss. 

	\begin{myDef}[WB Principle]
		Denote $i^*=\argmin_i (p_1,\cdots,p_n)$ . We say an objective function $L(p_1, p_2, \cdots, p_n)$  is $\gamma$-focused, if for any ${p'}_{i^*+1} > {p}_{{i^*+1}} \ge \gamma + p_{i^*}$ and ${p'}_{i^*-1} > {p}_{{i^*-1}} \ge \gamma + p_{i^*}$,  the following inequality holds: $|\frac{\partial L(\cdots, p_{i^*-1},p_{i^*},p_{i^*+1},\cdots)}{\partial p_{i^*}} |< |\frac{\partial L(\cdots, p'_{i^*-1},p_{i^*},p'_{i^*+1},\cdots)}{\partial p_{i^*}} |$.
\end{myDef}

This principle suggests that if the prediction probability of a token is small and its adjacent tokens are of large probabilities, the absolute value of the gradient of the loss function with respect to this token increases when enlarging the probability of its adjacent tokens. Such a principle guarantees that the optimization will focus on tokens that are not well optimized. 	

Note that several previous works actually satisfy the PS principle, although they do not explicitly define the principle. For example, curriculum learning and self-paced learning~\cite{bengio2009curriculum,pentina2015curriculum} tend to pay less attention to easy training samples according to handcraft rules or model outputs, but they fail to satisfy the WB principle.

As the MLE loss is the sum of log probabilities of tokens across positions and is a strictly decreasing function of $p_i$, it is easy to check that the MLE loss does not satisfy the above two principles and thus will not lead to smooth per-token probabilities. The proof is provided in supplementary materials (part B).

\subsection{The Proposed Loss}
\label{definition}
Our new loss is the combination of the MLE loss and a smooth regularization term.


For a sequence with length $n$, we can obtain $n-k+1$ subsequences using a sliding window with size $k$. We first consider the smoothness within each subsequence and then compute the overall smoothness for the whole sequence. In detail, for each sub-sequence $(p_i, \cdots, p_{i+k-1})$, we define its smoothness $r_i$ as the L2-norm of the distance between adjacent probabilities in the subsequence: 
				$r_i = \sqrt{\sum_{j=0}^{k-2} (p_{i+j} - p_{i+j+1})^2}$.
If the distance between adjacent probabilities is large, $r_i$ will be large. We further introduce a weight factor $w_i$ for the smoothness of the subsequence in each sliding window: 
      \begin{equation}
          \begin{aligned}
              \label{eq:function}
              &w_i = 1 - \min(p_i, \cdots, p_{i+k-1}).\\
          \end{aligned}
      \end{equation}	
Intuitively, $w_i$ up weights the subsequence with small probabilities, and down weights subsequence with large probabilities.

Then our new loss is designed as follow:
      \begin{equation}
          \begin{aligned}
              \label{eq:loss}
              L_{smo} = \frac{1}{n}{(-\sum_{i=1}^{n}{log \, p_i} + \lambda \sum_{i=1}^{n-k+1}{w_i r_i})},
          \end{aligned}
      \end{equation} 
where the first term is the MLE loss, the second term is the proposed smooth regularization term, which is a weighted average over the smoothness of the subsequence in each sliding window, and $\lambda$ is a hyper-parameter controlling the trade-off between the MLE loss and the regularization term. 
    
\subsection{Analysis}
\label{analysis}
In this subsection, we show that our proposed loss $L_{smo}$ satisfies the two principles in Section~\ref{2-condition}. 
	
	\begin{myThe}
		\label{local-penalty}
		$L_{smo}$ has the following properties:
            When $k=2$, $L_{smo}$ is $(\beta, \frac{1}{\lambda\beta}+\beta-1)$-saturated.
            When $k > 2$, $L_{smo}$ is $(\beta, 
            \sqrt{\frac{1}{2\sqrt{k}\beta\lambda}})$-saturated. 
	\end{myThe}
\begin{proof} 
The theorem is equivalent to show $\frac{\partial L_{smo}}{\partial p_i} \ge 0$ when $\max(p_{i-1}, p_{i+1}) \le \beta-\epsilon$ and $p_i \ge \beta$. 

For $k=2$ case, When  $i^{*} > 0$, we have $\frac{\partial L_{smo}}{\partial p_i}=\frac{1}{n}(-\frac{1}{p_i}+\lambda(2-p_{i-1}-p_{i+1}))$. When $p_i \ge \beta$ and $\max(p_{i-1}, p_{i+1}) \le \beta-\epsilon = 1-\frac{1}{2\lambda\beta}$, we have $\frac{\partial L_{smo}}{\partial p_i} \ge -\frac{1}{\beta}+\lambda\frac{1}{\lambda\beta}=0$. Then $L_{smo}$ is $(\beta, \frac{1}{2\lambda\beta}+\beta-1)$-saturated. When  $i^{*} = 0$, we have $\frac{\partial L_{smo}}{\partial p_i}=\frac{1}{n}(-\frac{1}{p_i}+\lambda(p_i-p_{i+1}))$. Given $p_i \ge \beta$ and $p_{i+1} \le \beta-\epsilon = 1-\frac{1}{\lambda\beta}$, we have $\frac{\partial L_{smo}}{\partial p_i} \ge -\frac{1}{\beta}+\lambda\frac{1}{\lambda\beta}=0$. Then $L_{smo}$ is $(\beta, \frac{1}{\lambda\beta}+\beta-1)$-saturated.

For general $k$, we have: 
\small{\begin{equation}
\frac{\partial L_{smo}}{\partial p_i} = \frac{1}{n}(-\frac{1}{p_i} + \lambda \sum_{j=i-k+1}^{i}{w_j (2p_i-p_{i+1}-p_{i-1}) r_j^{-1}}).
\end{equation}}
As $r_j \le \sqrt{k}$ and $\epsilon = \sqrt{\frac{1}{2\sqrt{k}\beta\lambda}}$, we have: 
\small{\begin{equation}
\begin{aligned}
\frac{\partial L_{smo}}{\partial p_i}  
& = \frac{1}{n}[\frac{1}{p_i} +\lambda \sum_{j=i-k+1}^{i}{w_j (2p_i-p_{i+1}-p_{i-1}) r_j^{-1}}]
\\ & \ge \frac{1}{n}[-\frac{1}{\beta} + \lambda \sum_{j=i-k+1}^{i}{(1-\beta+\epsilon)(2p_i-p_{i+1}-p_{i-1}) r_j^{-1}}]
\\ &  \ge \frac{1}{n}[-\frac{1}{\beta} + \lambda k (1-\beta+\epsilon)2\epsilon \sqrt{\frac{1}{k}}].
\end{aligned}
\end{equation}}

It is easy to check that $(1-\beta+\epsilon)\epsilon \ge \frac{\sqrt{k}}{2k\lambda\beta}$. Therefore, $ \frac{\partial L_{smo}}{\partial p_i} \ge 0$. 
The theorem follows. 
\end{proof}  

We provide a more intuitive explanation of the principle. The corollary below shows how the loss function works when the prediction probability is $\ge 0.5$.

\begin{myCor}
\label{beta0.5}
	$ L_{smo}$ is $(\frac{1}{2}, \frac{1}{\lambda}-\frac{1}{2})$-saturated when $k=2$, and is $(\frac{1}{2}, \sqrt\frac{1}{\lambda\sqrt{k}})$-saturated when $k>2$.
\end{myCor}
Corollary~\ref{beta0.5} shows a special case when $\beta=0.5$. When $p_i \ge \beta = 0.5$, the token will be correctly predicted, since the probability 0.5 is larger than the sum of the remaining probabilities in the softmax output. In such a situation, when $k=2$, once the prediction probability of a nearby token is smaller than $\frac{1}{\lambda} - \frac{1}{2}$, the loss will not continue optimizing $p_i$ to a large value. 
	
	\begin{myThe}
		\label{local-accelerate}
        $L_{smo}$ has the following properties:
 			When $k=2$, $L_{smo}$ is $0$-focused.
            When $k > 2$, $L_{smo}$ is $0.193$-focused. 
	\end{myThe}

\begin{proof} 
Denote $L_{smo}'=L_{smo}(\cdots, p'_{i^*-1},p_{i^*},p'_{i^*+1},\cdots)$.

For $k=2$ case, when $i^{*} \geq k$, by definition, we have $\frac{\partial L'_{smo}}{\partial p_{i^*}}=\frac{1}{n}[-\frac{1}{p_{i^*}}+\lambda(2p_{i^*}-p'_{i^*-1}-p'_{i^*+1}) ]< \frac{1}{n}[-\frac{1}{p_{i^*}}+\lambda(2p_{i^*}-p_{i^*-1}-p_{i^*+1})]=\frac{\partial L_{smo}}{\partial p_{i^*}}$, and $\frac{\partial L_{smo}}{\partial p_{i^*}}<0$. when $k = 2, i^{*} < k$, since $2p_{i^*}-p'_{i^*+1}-1 < 2p_{i^*}-p_{i^*+1}-1$ and $2p_{i^*}-p'_{i^*-1}-1 < 2p_{i^*}-p_{i^*-1}-1$(if $i^*-1$ exists), we have $\frac{\partial L'_{smo}}{\partial p_{i^*}} < \frac{\partial L_{smo}}{\partial p_{i^*}}<0$. Therefore, the loss is $0$-focused. 

For general $k$, similar to the proof of Theorem 1,  we have $\frac{\partial L_{smo}}{\partial p_{i^*}} = \frac{1}{n}[-\frac{1}{p_{i^*}} + \lambda \sum_{j=i^*-k+1}^{i^*} \{-r_j+(1-p_{i^*}) (2p_{i^*}-p_{i^*+1}-p_{i^*-1}) r_j^{-1}\}]$. Denote $x=p_{i^* + 1}-p_{i^*}, y=p_{i^* - 1}-p_{i^*}, a= r_j^2-x^2-y^2, c = 1- p_{i^*}$, we have $r_j-(1-p_{i^*}) (2p_{i^*}-p_{i^*+1}-p_{i^*-1}) r_j^{-1}=c\frac{x+y}{\sqrt{x^2+y^2+a}}+\sqrt{x^2+y^2+a}=f(x)$. Without lose of any  generality, we treat this term as a function of $x$. Then, it is sufficient to show $f(x)$ is strictly monotone increasing when $p_{i^* \pm 1} \ge \gamma + p_{i^*}$.
As $p_{i^* \pm 1} \ge 0.193+p_{i^*}$, we have $y>0.193c$ and $x>0.193c$. Using Cauchy-Schwarz Inequality, we have $f^{'}(x) \ge \frac{(x^3+y^2x-cyx+cy^2)}{(x^2+y^2+a)^{-\frac{3}{2}}} \ge \frac{(3\sqrt[3]{cx^4y^4}-cyx)}{(x^2+y^2+a)^{-\frac{3}{2}}} \ge \frac{(3cyx\sqrt[3]{0.193^2}-cyx)}{(x^2+y^2+a)^{-\frac{3}{2}}} > 0$. Therefore,  $\frac{\partial L_{smo}'}{\partial p_{i^*}} < \frac{\partial L_{smo}}{\partial p_{i^*}}$.

The theorem follows.
\end{proof} 


Theorem~\ref{local-penalty} and Theorem~\ref{local-accelerate} together theoretically study our proposed loss $L_{smo}$. In the next section, we will empirically show that $L_{smo}$ can ensure smooth per-token probabilities and improve the performance of sequence to sequence learning.

	\section{Experiments}
    \label{experiments}
	We test our method on three translation tasks as well as a text summarization task. We first introduce the datasets and experiment settings. Then we report the results of our method and compare it with several baselines. We also conduct some analyses to investigate how our proposed loss works.
    In this section, we mark our proposed loss as \emph{Smooth} for simplicity.
	\subsection{Datasets}
\noindent{\textbf{WMT14 English-German Translation}} Following the setup in~\cite{luong2015effective,gehring2017convolutional,vaswani2017attention}, the training set of this task  consists of 4.5M sentence pairs. Source and target sentences are encoded by 37K shared sub-word types based on byte-pair encoding (BPE)~\cite{sennrich2015neural}. We use the concatenation of newstest2012 and newstest2013 as the validation set and newstest2014 as the test set. 
	
\noindent{\textbf{WMT17 Chinese-English Translation}} We filter the full dataset of 24M bilingual sentence pairs by removing duplications and get 19M sentence pairs for training. The source and target sentences are encoded using 40K and 37K BPE tokens. We report the results on the official newstest2017 test set and use newsdev2017 as the validation set. 
	
\noindent{\textbf{IWSLT14 German-English Translation}} This dataset contains 160K training sentence pairs and 7K validation sentence pairs following~\cite{cettolo2014report}. Sentences are encoded using BPE with a shared vocabulary of about 37K tokens. We use the concatenation of dev2010, tst2010, tst2011 and tst2011 as the test set, which is widely adopted in~\cite{huang2017neural,bahdanau2016actor}.
	
\noindent{\textbf{Text Summarization}} We use the English Gigaword Fifth Edition~\cite{graff2003english} corpus for text summarization, which contains 10M news articles with the corresponding headlines. We follow ~\cite{rush2015neural,shen2016neural} to filter the article-headline pairs and get roughly 3.8M pairs for training and 190K pairs for validation. We test our method on the widely used Gigaword test set with 2K article-title pairs~\cite{shen2016neural,suzuki2017cutting}.
\begin{table}[h]
		\begin{center}
			\begin{tabular}{llll}
				\toprule 
				\bf Dataset & \bf Setting & \bf Method & \bf BLEU \\
				\hline
				\multirow{6}{*} {En$\to$De} & \multirow{3}{*} {Transformer Base} & MLE & 27.30* \\
				& & Curriculum & 27.41 \\
				& & Smooth & \bf 28.32 \\
				\cline{2-4}
				& \multirow{3}{*} {Transformer Big} & MLE &  28.40* \\
				& & Curriculum & 28.22 \\
				& & Smooth & \bf29.01 \\
                \hline
                \multirow{3}{*} {Zh$\to$En} & \multirow{3}{*} {Transformer Big} & MLE & 24.30\\
                & & Curriculum & 24.06  \\
                & & Smooth & \bf 25.05 \\
                \hline
                \multirow{3}{*} {De$\to$En} & \multirow{3}{*} {Transformer Small} &MLE & 32.20 \\
                & & Curriculum & 31.97 \\
                & & Smooth & \bf 33.47\\
				\bottomrule
			\end{tabular}
		\end{center}
        \vspace{-10pt}
		\caption{\label{Result-table} BLEU scores on the three NMT tasks. Our method is denoted as ``Smooth” and the curriculum learning method is denoted as ``Curriculum''. BLEU scores marked with * are from ~\cite{vaswani2017attention}.}
	\end{table}
\begin{table}[t!]
  		\begin{center}
			\begin{tabular}{llc}
				\toprule 
				\bf  Dataset &  \bf  Method & \bf  BLEU \\
				\hline
				\multirow{7}{*} { En$\to$De} &  Local Attention~\cite{luong2015effective} &  20.90 \\
				&  ByteNet~\cite{kalchbrenner2016neural} &  23.75 \\
				&  ConvS2S~\cite{gehring2017convolutional} &  25.16 \\
				\cline{2-3}
				&  Transformer Base ~\cite{vaswani2017attention}&  27.30 \\
				&  Transformer Big~\cite{vaswani2017attention}&   28.40 \\
				\cline{2-3}
				&  Transformer Base+Smooth &   28.32 \\
				&  Transformer Big+Smooth &  \bf 29.01 \\
                \hline
                \multirow{3}{*} {Zh$\to$En} &  SougoKnowing~\cite{wang2017sogou} &  24.00 \\
                &  xmunmt ~\cite{tan2017xmu} &  23.40  \\
                \cline{2-3}
                &  Transformer+Smooth & \bf 25.05\\
                \hline
                \multirow{4}{*} {De$\to$En} &  Actor-critic~\cite{bahdanau2016actor}  &  28.53 \\
                & CNN-a \cite{gehring2016convolutional}  &  30.04\\
                &  Dual transfer learning ~\cite{Wang2018Dual} &  32.35\\
                 \cline{2-3}
                &  Transformer+Smooth &  \bf 33.47 \\
				\bottomrule
			\end{tabular}
		\end{center}
        \vspace{-10pt}
		\caption{\label{Baseline-table} Comparison with previous works.}
	\end{table}
    \vspace{-10pt}
	\subsection{Implementation Details} 
    
We choose Transformer~\cite{vaswani2017attention} as our basic model and use \texttt{tensor2tensor}~\cite{DBLP:journals/corr/abs-1803-07416}, which is an encoder-decoder structure purely based on attention and achieves state-of-the-art accuracy on several sequence to sequence learning tasks. 

For WMT14 English-German task, we choose both \emph{transformer\_base} and \emph{transformer\_big} configurations following ~\cite{vaswani2017attention}. Both the configurations have a 6-layer encoder and a 6-layer decoder, with 512-dimensional and 1024-dimensional hidden representations, respectively. For WMT17 Chinese-English task, we use \emph{transformer\_big} considering the large scale of the training set. For IWSLT14 German-English task, we choose \emph{transformer\_small} with a 2-layer encoder/decoder and 256-dimensional hidden representations. For the text summarization task, we choose \emph{transformer\_small} with a 4-layer encoder/decoder and 256-dimensional hidden representations. Unless otherwise stated, we keep all other hyper-parameters the same as ~\cite{vaswani2017attention}. Unless otherwise stated, in all our experiments, the sliding window $k$ of subsequences in our loss is set to $4$ and $\lambda$ is set to $0.6$ according to the validation performance.

We choose Adam optimizer with $\beta_{1}= 0.9$, $\beta_{2} = 0.98$, $\varepsilon = 10^{-9}$. We follow the learning rate schedule in ~\cite{vaswani2017attention}. During inference, for WMT14 English-German and WMT17 Chinese-English tasks, we use beam search with beam size $6$ and length penalty $1.0$~\cite{shen2016neural}. For IWSLT14 German-English task and the text summarization task, we set beam size to $6$ and length penalty to $1.1$. These hyper-parameters are chosen after experimentation on the development set. 
	
Our method is to train the Transformer model by minimizing the loss defined in Equation~\ref{eq:loss}. We compare our method with the MLE loss based method. We also compare with a token-level curriculum learning~\cite{bengio2009curriculum,pentina2015curriculum} method. The strategy for selecting examples in this curriculum learning is as follows. As the training process goes on, this baseline tends to focus on hard tokens: It ignores target tokens with probabilities larger than $0.85$, and only updates its model using tokens with probabilities smaller than $0.85$. All the three methods use the same model configurations for fair comparisons.

We use the ROUGE F1 score~\cite{ROUGE}\footnote{Calculated by https://github.com/pltrdy/rouge} to measure the text summarization quality. For fair comparisons, we use the tokenized case-sensitive BLEU~\cite{papineni2002bleu}\footnote{https://github.com/moses-smt/mosesdecoder/blob/master/scripts/ generic/multi-bleu.perl} to measure the translation quality for WMT14 En$\to$De, scareBLEU~\footnote{https://github.com/awslabs/sockeye/tree/master/contrib/sacrebleu} for WMT17 Zh$\to$En, and tokenized case-insensitive BLEU for IWSLT14 De$\to$En. For both the two metrics, the bigger the better. We use the validation performance to choose a good $\lambda$ and set $k$ to 4.
	
\subsection{Results of Machine Translation}
    
We compare our method with the corresponding MLE and curriculum learning baselines, and then with the results reported in several previous works. We denote WMT14 English-German task as En$\to$De, WMT17 Chinese-English task as Zh$\to$En and IWSLT14 German-English task as De$\to$En. 

The BLEU scores are listed in Table ~\ref{Result-table}. As can be seen, the curriculum learning method does not significantly outperform the MLE baseline. This result suggests that only focusing on hard instances and not meeting the WB principle (instead of smoothing per-token probabilities) is not enough for sequence to sequence learning.
For WMT14 En$\to$De task, our method with the \emph{transformer\_base} model  reaches 28.32 BLEU score, i.e.,  1.02 BLEU score improvement over the MLE loss in ~\cite{vaswani2017attention}. We also achieve 0.61 BLEU score improvement from 28.40 to 29.01 on the \emph{transformer\_big} model. For WMT17 Zh$\to$En and IWSLT14 De$\to$En tasks, our method improves the BLEU score from 24.30 to 25.05 and 32.20 to 33.47, respectively. 
	
We further list the numbers reported in previous works on the three tasks in Table~\ref{Baseline-table}. Note that those works use different model structures, including LSTM~\cite{hochreiter1997long}, CNN~\cite{gehring2017convolutional}, and Transformer~\cite{vaswani2017attention}. For WMT14 En$\to$De task, our method achieves 29.01 BLEU score, outperforming the best number formally published so far and setting a new record on this dataset. For WMT17 Zh$\to$En task, our method beats the champion\footnote{We choose the best single model for comparison, as the best result is obtained by combining multiple models (http://matrix.statmt.org/matrix/systems\_list/1878).} of WMT17 Zh$\to$En challenge by 1.05 BLEU score. For IWSLT14 De$\to$En task, our method also achieves 1.12 BLEU score improvement. 

For all above tasks, we have done significant test, and the improvements are significant where all p-value $<$ 0.05. The improvements on the three tasks demonstrate the effectiveness of our proposed sentence-wise smooth regularization for sequence to sequence learning.   
\begin{figure}[htbp]
\centering
\includegraphics[width=7.5cm]{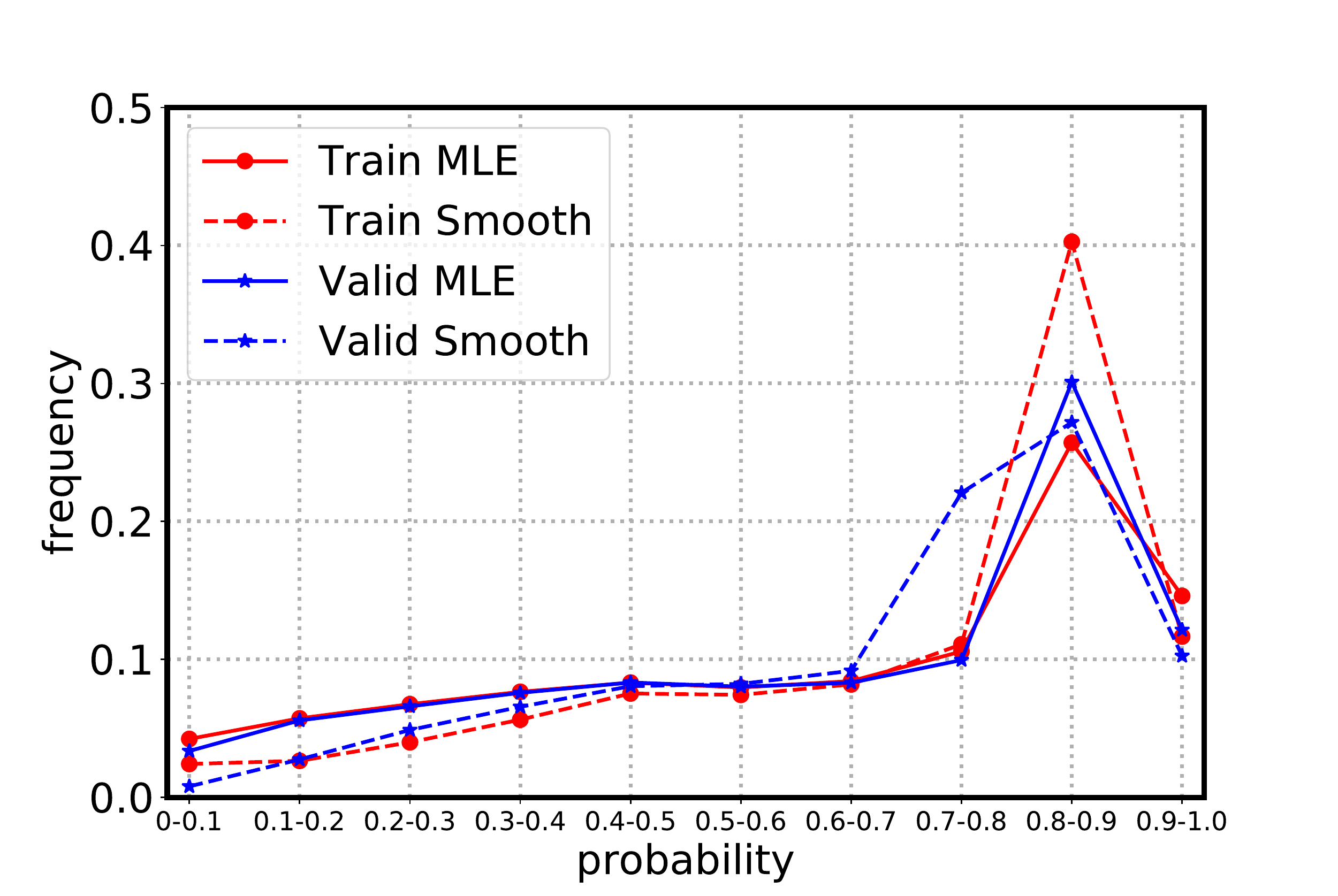}
\caption{Comparison of the per-token probability frequency between our proposed method and the MLE baseline in \emph{transformer\_base} on WMT14 English-German train and validation set. The red lines show the results on the training set and the blue lines show the results on the validation set.} 
\label{fig:2}
\end{figure}
\subsection{Method Analysis}
In this subsection, we conduct some analyses to better understand our proposed method.

\noindent{\textbf{Distribution of per-token probabilities}}
We first analyze how the distribution of the per-token probabilities changes when applying our proposed method. We conduct this study on WMT14 English-German training and validation set. We divide the probability interval $[0,1]$ into ten buckets. For both our proposed loss and the MLE baseline, we calculate the frequency of the per-token probabilities in each bucket, as shown in Figure~\ref{fig:2}. From the figure, we can see that our method reduces the frequency of both small (smaller than 0.5 in the figure) and very big (larger than 0.9 in the figure) probabilities while greatly increases the frequency of the probabilities located in $[0.7, 0.9]$. This demonstrates that our method indeed makes the per-token probabilities smoother. We further study the changes in loss and list them in supplementary materials (part C).  
Besides, we have analyzed the variance changes and found that the variance becomes smaller after applying our method on IWSLT14 De-En test set. The sentences originally with high variance get lower variance and higher BLEU score improvements. Specifically, we divided sentences into 5 groups according to their variance in the MLE training: the group with highest variance gets 1.47 BLEU score improvements on average in our method, while the sentence with lowest variance has 0.35 BLEU score improvements.
\begin{table}[h]
\centering
\begin{tabular}{cccccc}
	\toprule 
	$\lambda$ &0.4 & 0.6& 0.8& 1.0& 2.0\\ 
    \hline
	\bf BLEU & 33.19 & 33.47 & 33.42 &33.39 & 33.26 \\
    \bottomrule
\end{tabular} 
\caption{\label{Sens-table} The BLEU scores with respect to different $\lambda$ on IWSLT De$\to$En task.}
\end{table}
\begin{table}
\centering
\begin{tabular}{ccccc}
	\toprule
	Smooth & $\times$ & $\times$ & $\surd$ & $\surd$\\
    \hline
    label smoothing & $\times$ & $\surd$ & $\times$ & $\surd$\\ 	
    \hline
	BLEU & 31.88 & 32.20 & 33.04 & 33.47 \\
    \bottomrule
\end{tabular}
\caption{\label{Ablation-table} Ablation study on label smoothing and our method on IWSLT De$\to$En task.}
\end{table}

    	\begin{table*}[!htbp]
		\centering
		\begin{tabular*}{\linewidth}{ll}
			\toprule[1.5pt]
			\bf Source & ich habe ihnen ein beispiel mitgebracht von dem nordamerikanischen stamm der nootka-indianer . \\
			\bf Reference & i brought along an example of the north american tribe of the nootka indians . \\
			\cmidrule[0.1pt]{1-2}
			\bf MLE & i've brought you an example of the north american tribe .\\
			\bf Smooth&  i brought you an example of the north american tribe of the nootka indian . \\
			\toprule[1.5pt]
			\bf Source & also wir haben vorne die augen und jetzt denken sie , wir sehen von vorne nach hinten . \\
			\bf Reference & so we have our eyes at the front and now you think we see from front to back . \\
			\cmidrule[0.1pt]{1-2}
			\bf MLE &so we've got the front , and now you're thinking , we're looking forward from the back . \\
			\bf Smooth& so we've got the eyes forward , and now you think , we're going to look back from the front . \\
			\toprule[1.5pt]
			\bf Source &da konnten dann die besucher in den raum reingehen, sich auf diesen stuhl setzen, sich diesen helm aufsetzen. \\
			\bf Reference & visitors could go into the room , sit down on this chair , and put on this helmet . \\
			\cmidrule[0.1pt]{1-2}
			\bf MLE & and then the visitor in the room could put themselves on this chair , put it on that helmet . \\
			\bf Smooth & then the visitor could go into the room , sit down on that chair , put that helmet on . \\
			\bottomrule[1.5pt]
		\end{tabular*}
		\caption{\label{case_study} Examples of the translation results by our proposed method ("Smooth") and the MLE baseline ("MLE") on IWSLT14 De$\to$En test set.}
        \vspace{-15pt}
	\end{table*}

\noindent{\textbf{Hyper-parameter $\lambda$}} is used to control the effect of our smooth regularization. In previous experiments, we use the validation performance to choose a good $\lambda$. Here we study how $\lambda$ impacts the performance of our method. We train a group of models with different $\lambda$ on IWSLT De$\to$En task. The results are listed in Table~\ref{Sens-table}. It can be seen that too small or too big $\lambda$ will hurt the performance of our method. 
This is consistent with Theorem~\ref{local-penalty} and~\ref{local-accelerate} that too small $\lambda$ cannot ensure the regularization effect, while too large $\lambda$ overemphasizes the smoothness. We also study the influence of another hyper-parameter: the sliding window $k$, and list the results in supplementary materials (part D).

\noindent{\textbf{Label Smoothing}}~\cite{Szegedy2015Rethinking,vaswani2017attention} is related to our work in the sense that it makes the true label less confident in training. While applying the cross-entropy loss to a classification task, it is usually expected to output probability 1 for true labels and 0 for the others. However, as the true labels are often the aggregated results of multiple annotators, there may be some disagreements among the annotators. Thus, label smoothing relaxes the target probabilities of true labels, e.g., assigning probability 0.9 instead of 1 and allocating 0.1 to other labels. Actually, the open-source code of Transformer has already implemented label smoothing, and the Transformer based methods (including both the baseline method and our method) adopt this technique in our experiments. 

It seems that label smoothing can also prevent the probability from being too large, and therefore we do an ablation study to compare our method with label smoothing. As shown in Table~\ref{Ablation-table}, only adding label smoothing improves 0.32 BLEU score, while adding our method alone can improve 1.16 BLEU score. Further adding our method to label smoothing achieves 1.27 BLEU score improvement. 
	

\noindent{\textbf{Case Study}}
In this subsection, we present several translation cases to demonstrate how our proposed method improves the sequence generation accuracy, as shown in Table~\ref{case_study}. We have following observations.
    
First, our method helps the model to generate more accurate and fluent sequence. For example, in the second case, our method catches the information `look from front to back', which is misunderstood by the MLE baseline. In the third case, the MLE baseline fails to recognize `put on helmet' and `sit on chair', which are correctly translated by our method.
    
Second, rare words, prepositions and conjunctions, usually skipped/ignored by the MLE baseline, are well captured by our method. For example, the phrases `of the nootka indian' and `have brought' in the first case and the token `eyes' in the second case are generated by our method but missed by the baseline method. These tokens are usually of low probabilities predicted by the baseline model and are likely to be neglected in the beam search process.
	\begin{table}[hbtp]
		\begin{center}
			\begin{tabular}{lccc}
				\toprule 
                \bf Model & \bf  R-1 & \bf  R-2 & \bf  R-L\\ 
                \hline
				RNN MRT\cite{shen2016neural} & 36.54 & 16.59 & 33.44\\
				WFE\cite{suzuki2017cutting} & 36.30 & 17.31 & 33.88 \\
				ConvS2S\cite{gehring2017convolutional} & 35.88 & 17.48 &33.29 \\
                \hline
            	 Transformer &35.34 & 16.44 & 32.71\\
                 Transformer+Curriculum& 35.61 &16.36 & 32.78\\
                 Transformer+Smooth&36.32 &17.17 &33.54\\
				\bottomrule
			\end{tabular}
		\end{center}
        \vspace{-10pt}
		\caption{\label{SUM-table} The ROUGE-1/2/L scores for different methods on the text summarization task.}
        \vspace{-15pt}
	\end{table}	
\subsection{Results of Text Summarization}
Table~\ref{SUM-table} shows the results of the text summarization task. Compared with the MLE loss, our method improves the ROUGE-1/2/L by 1.0, 0.7, 0.9 points. Besides, our method also outperforms the curriculum learning baseline. Compared with other previous works including minimize risk training~\cite{shen2016neural,shen2015minimum},  WFE~\cite{suzuki2017cutting} that cutting off redundant repeating, our method achieves comparable results. These methods have different motivations from ours. Therefore, note that these works are orthogonal to our method and we expect to obtain more improvements by combining with these methods.

\section{Conclusion and Future Work}
\label{conclusion}
In this paper, we observed that in sequence to sequence learning, sentences with non-smooth probabilities are usually of low quality. We design a new loss with a sentence-wise smooth regularization term to solve this problem. Experimental results show that our method outperforms the conventional MLE loss and achieves state-of-the-art BLEU scores on the three machine translation tasks.
There are many directions to explore in the future. First, we focused on machine translation and text summarization in this work. We will apply our method to more sequence generation tasks, such as image captioning and question answering. 
Second, we will investigate if there exist better methods to ensure the probability sequence to be smooth. 
Third, the smoothness of token probabilities is motivated from empirical observations in this work. It is interesting to study the smoothness of loss from a theoretical perspective, e.g., why a smooth loss function is preferred, and what it brings to the learning method.

	
\bibliographystyle{aaai}
\bibliography{aaai}

\appendix
\section{Empirical Study}
In Section 1, we have done an empirical study to analyze the relationship between the smoothness of per-token probabilities and BLEU scores, as well as the components of low-probability tokens. In this section, we will describe the experiments and results in details. 

For Figure 1(a), we separate sentences by MLE score to clearly demonstrate the relationship between the smoothness and the BLEU, since the comparison of variance is meaningless if sentences’ losses are not in similar value. We only describe the relationship between the BLEU score and variance in Figure 1(a), and therefore we list some other detailed information here in case the readers want to know. The different loss groups contain a similar number of sentences, except for \emph{loss 0.2-0.3} which contains slightly more sentences. Different variance groups contain different number of sentences:  \emph{variance 0.02-0.03} group contains about $40\%$ sentences,  \emph{variance 0.03-0.04} and \emph{variance 0.01-0.02} group together contain about $40\%$ sentences. The case-insensitive tokenized BLEU scores are computed by the same scripts described in Section 4.

For Figure 1(b), all the POS tags are obtained through NLTK (https://www.nltk.org/). We describe the meaning of the first several POS tags in Figure 1(b) and more information can be found from https://github.com/nltk/nltk. ``NN'' means \emph{noun}, ``DT'' means \emph{determiner}, ``PRP'' means \emph{personal pronoun}, ``CC'' means \emph{conjunction, coordinating} and ``IN'' means \emph{preposition or conjunction, subordinating}. 

We conduct the POS tags of all the tokens as well as the tokens with low probabilities, and find that tokens with low probabilities contain 1.5 times more conjunction and pronouns, with less adjective and nouns. 

We also find that tokens with low probability contain the similar ratio of rare words with that in the whole training set, and the non-smooth problem is not mainly caused by rare words. It can be intuitively explained that the non-smooth in MLE is not only caused by non-smooth token distribution, but also non-smooth 2/3/4/n-gram distribution. The context information influences the learning of tokens. For example, the context of conjunction and pronouns words can be very different in different sentences, which makes the learning of these words harder and yields low probabilities.

\section{Analysis of MLE and Other Methods on the Two Principles}
In Section 3, we claim that MLE loss does not satisfy the proposed two principles. Here we give a simple discussion and proof.

For MLE loss $L_{mle} = -\sum_{i=1}^{n}{\log p_i}, p_i \in (0, 1]$, it always meets $\frac{\partial L_{mle}}{\partial p_i} <= 0$, which means bigger $p_i$ will never be penalized. Therefore, it does not meet the PS principle. Besides, we have $\frac{\partial^{2} L_{mle}}{\partial p_i \partial p_j} = 0$. It means that MLE does not meet the WB principle.

In Section 3, we also claim that some related methods do not satisfy the proposed principles. Here we give a simple discussion and proof about label smoothing.

Label smoothing changes the cross entropy loss for every word. Therefore, larger $p_i$ will be penalized. However, $p_i$ will not be penalized according to the difference of $p_{i \pm 1}$ and $p_i$. Therefore, it does not always meet PS principle. Besides, we have $\frac{\partial^2 L_{ls}}{\partial p_{i} \partial p_j} = 0$. It means label smoothing does not meet the PS principle.

\section{Hyper-parameter k}
\begin{table}[htbp]
\centering
\begin{tabular}{ccccc}
	\toprule 
	$k$ & 2 & 3& 4& 6\\ 
    \hline
	\bf BLEU & 33.05 & 33.34 & 33.47 & 33.41 \\
    \bottomrule
\end{tabular} 
\caption{\label{k-table} The BLEU scores with respect to different $\lambda$ on IWSLT De$\to$En task.}
\end{table}
We study how sliding window $k$ impacts the performance of our method. We train a group of models with different $k$ on IWSLT De$\to$En task, with $\lambda = 0.6$. The results are listed in Table~\ref{k-table}. It can be seen that the performance is not that sensitive according to different $k$ , and too small or too big $k$ will hurt the performance of our method slightly. 

\section{Temperature Softmax} 
Considering temperature is often used for decreasing noise, In IWSLT14 De$\to$En, we also compare to a baseline, which adds a temperature parameter $\tau$ to the Softmax layer as $\text{sigmoid}(Wx/\tau)$. We notice that, first, using a large temperature $\tau$ during training will not lead to a ``smoother'' model as such trick is equivalent to rescale the Softmax parameters to a smaller range in initialization. To be concrete, $\text{sigmoid}(Wx/\tau)$ is computationally equivalent to $\text{sigmoid}(W^\prime x)$ by setting $W^\prime=W/\tau$. Second, during testing, we tried different temperature $\tau$ other than 1.0, which is also referred as \emph{model calibration}. However, we empirically find that it harms the translation performance, making the BLEU measure drops a lot (more than 5 points). 

\end{document}